\documentclass{article}
\usepackage{microtype}
\usepackage{graphicx}
\usepackage{subcaption}
\usepackage{booktabs}
\usepackage{hyperref}

\usepackage[accepted]{icml2026}
\usepackage{amsmath}
\usepackage{amssymb}
\usepackage{mathtools}
\usepackage{amsthm}
\usepackage[capitalize,noabbrev]{cleveref}
\theoremstyle{plain}
\newtheorem{theorem}{Theorem}[section]
\newtheorem{proposition}[theorem]{Proposition}

\theoremstyle{definition}

\theoremstyle{remark}
\newtheorem{remark}[theorem]{Remark}
\usepackage[textsize=tiny]{todonotes}
\icmltitlerunning{Creative Collision: Directorial Persona Steering in LLMs}
\begin{document}
\twocolumn[
\icmltitle{Creative Collision: Directorial Persona Steering\\
           and Competition in Large Language Models}
\icmlsetsymbol{equal}{*}
  \begin{icmlauthorlist}
    \icmlauthor{Subramanyam Sahoo}{horizon}
    \icmlauthor{Justin Shenk}{horizon}
  \end{icmlauthorlist}
  \icmlaffiliation{horizon}{Horizon Research}
  \icmlcorrespondingauthor{Subramanyam Sahoo}{sahoo2vec@gmail.com}
  \icmlkeywords{LoRA, fine-tuning, weight steering, symmetry, emergence, LLM}
  \vskip 0.3in
]
\printAffiliationsAndNotice{}

\begingroup
\renewcommand\thefootnote{}\footnotetext{Code available at: \url{https://github.com/SubramanyamSahoo/Creative-Collision}}
\footnotetext{This work was conducted as part of the AI Safety Camp 2026 edition.}
\endgroup

\begin{abstract}
Activation steering has emerged as a powerful tool for shaping the behaviour
of large language models at inference time, yet most prior work injects a
\emph{single} semantic direction into the residual stream. We study the
richer setting in which two semantically opposing steering vectors are
superimposed---a regime we call \textbf{Creative Collision}. Concretely, we
construct directorial persona vectors for Steven Spielberg (optimistic,
redemptive moral valence) and Martin Scorsese (dark, morally ambiguous) via
mean-difference activation contrast on curated screenplay-derived corpora,
then interpolate between them with a scalar mixing parameter
$\alpha \in [0,1]$ and a steering coefficient $\lambda$. Across five
evaluation axes---moral valence, generation coherence, surface style,
directional dominance, and vector geometry---three principal findings emerge:
(i)~Spielberg's representational signature exhibits robust
\emph{directional dominance}, suppressing Scorsese's moral influence across
almost the entire interpolation range; (ii)~intermediate collision points
paradoxically \emph{improve} generation coherence relative to pure
single-director steering at high $\lambda$; and (iii)~both personas localise
maximally to layer~28 of a 40-layer decoder-only transformer, revealing a
shared \emph{moral-tone substrate}. These results illuminate the geometry of
competing semantic directions in transformer residual streams and have direct
implications for controllable creative generation and value-aligned narrative
synthesis.
\end{abstract}

\section{Introduction}
\label{sec:intro}

The residual stream of a modern transformer encodes far more than syntactic
structure and factual knowledge; high-level semantic properties such as
sentiment, formality, and moral tone are represented as approximately linear
directions in activation
space~\citep{zou2023representation,burns2023discovering,elhage2022toy}.
\emph{Activation steering} exploits this linearity by adding a learned
direction to the hidden states of selected layers, modifying model behaviour
without any weight update~\citep{turner2023activation,li2023inference}.
Despite successful applications to truthfulness~\citep{li2023inference},
safety~\citep{zou2023representation}, and
persona~\citep{rimsky2024steering}, the dynamics of \emph{multi-vector
collision}---what happens when two opposing directions compete for
representational control---remain largely unexplored.

Creative writing offers a natural and interpretable test bed for such
competition. The films of Steven Spielberg are characterised by redemptive
arcs, emotional catharsis, childhood wonder, and optimistic resolution; those
of Martin Scorsese foreground moral ambiguity, violence, betrayal, and the
self-destructive trajectories of flawed protagonists. These directors define
semantically rich, culturally legible, and stylometrically measurable poles
of a \emph{moral valence axis}, making them ideal anchors for opposing
steering directions.

We introduce \textbf{Creative Collision}: a framework for probing vector
competition between two directorial persona representations. We extract
Spielberg ($\mathbf{v}_\mathrm{SPL}$) and Scorsese
($\mathbf{v}_\mathrm{SCO}$) steering vectors via mean-difference contrast on
curated corpora, then form a parametric family of interpolated vectors
\begin{equation}
  \mathbf{v}_\alpha = (1-\alpha)\,\hat{\mathbf{v}}_\mathrm{SPL}
                    + \alpha\,\hat{\mathbf{v}}_\mathrm{SCO},
  \quad \alpha \in [0, 1].
  \label{eq:interp_intro}
\end{equation}
Applying $\mathbf{v}_\alpha$ to intermediate transformer layers at inference
time with coefficient $\lambda$, we systematically measure the effect on
moral valence, coherence, surface style, and the geometry of the underlying
collision.

\paragraph{Contributions.}
\begin{itemize}
  \item A curated screenplay-derived contrast corpus and mean-difference
        extraction pipeline for directorial persona steering vectors
        (\S\ref{sec:method}).
  \item An empirical characterisation of \emph{moral valence interpolation}
        showing that linearly mixing opposing semantic directions produces
        non-monotone, non-linear effects on moral tone in generated text
        (\S\ref{sec:valence}).
  \item The discovery of \emph{directional dominance}: Spielberg's
        representational signature suppresses Scorsese's moral influence for
        all $\alpha \leq 0.75$, robust across all tested $\lambda$
        (\S\ref{sec:dominance}).
  \item A formal norm-reduction proposition explaining the
        \emph{coherence valley} at intermediate collision points, with
        empirical support (\S\ref{sec:coherence}).
  \item A layer-wise localisation experiment identifying layer~28 as the
        shared moral-tone hub for both directors in a 40-layer decoder-only
        model (\S\ref{sec:layer}).
  \item A geometric analysis of the collision vector trajectory and its
        smooth cosine-similarity arc (\S\ref{sec:geometry}).
\end{itemize}

\section{Related Work}
\label{sec:related}

\paragraph{Activation steering.}
\citet{turner2023activation} introduced \emph{activation addition}, injecting
a fixed residual-stream offset to modify behaviour. \citet{zou2023representation}
scaled this approach with Representation Engineering, extracting linear probes
for high-level semantic properties from contrast pairs and using them for
monitoring and steering. \citet{li2023inference} targeted attention heads
rather than the full residual stream for truthfulness control via Inference-Time
Intervention~(ITI). \citet{rimsky2024steering} conducted a systematic study of
steering vector effectiveness across layers and magnitudes in Llama~2. We
extend this line by studying \emph{two competing directions} applied
simultaneously, a setting none of the above addresses.

\paragraph{Linear representations.}
The linear representation hypothesis holds that human-interpretable features
are encoded as approximately linear directions in
activation space~\citep{elhage2022toy,mikolov2013distributed}. \citet{burns2023discovering}
demonstrated linear encoding of truthfulness via Contrastive Consistent
Search~(CCS); \citet{templeton2024scaling} recovered interpretable
monosemantic features with sparse autoencoders. Our experiments implicitly
test whether a rich composite property---directorial moral persona---admits a
linear causal representation amenable to vector arithmetic \cite{sahoo2026linearprobesdetecttask}.

\paragraph{Persona and style in language models.}
Persona conditioning via prompting~\citep{brown2020language} and RLHF
fine-tuning~\citep{ouyang2022training} has been widely studied. At the
activation level, \citet{subramani2022extracting} showed that steering
vectors can reproduce specific stylistic fragments. Our focus is coarser---
creative moral tone rather than verbatim style---and specifically on what
happens when two such personas are simultaneously injected.

\paragraph{Creative generation.}
LLMs have been applied to narrative generation~\citep{yang2023doc} and
co-creative screenwriting~\citep{mirowski2023co}. Mechanistic studies of how
moral tone and stylistic identity are encoded and steered in creative
contexts remain scarce; our work aims to bridge activation-level
interpretability with creative generation evaluation.

\section{Method}
\label{sec:method}

\subsection{Directorial Contrast Corpus}

We construct a paired corpus
$\mathcal{D} = \{(x_i^\mathrm{SPL},\, x_i^\mathrm{SCO})\}_{i=1}^{N}$
where each pair is matched on narrative \emph{situation} (confrontation,
loss, moral choice) so that the contrast isolates directorial moral tone
rather than plot complexity.

\textbf{Spielberg passages} ($N_\mathrm{SPL} = 100$ excerpts, 80--200 tokens
each) are drawn from \emph{E.T.\ the Extra-Terrestrial}, \emph{Schindler's List}
(redemptive sequences), \emph{Jurassic Park}, \emph{Saving Private Ryan}
(self-sacrifice arcs), and \emph{Close Encounters of the Third Kind}.
These passages are characterised by emotional catharsis, childhood wonder,
moral courage, and optimistic resolution.

\textbf{Scorsese passages} ($N_\mathrm{SCO} = 100$ excerpts) are drawn from
\emph{Goodfellas}, \emph{Taxi Driver}, \emph{The Departed}, \emph{Raging Bull},
and \emph{Casino}. They foreground moral decay, violence, betrayal, and
protagonists trapped by self-destructive compulsion.

\subsection{Steering Vector Extraction}
\label{sec:extraction}

We use a 14-billion-parameter open-weight decoder-only transformer with 40
layers~\citep{qwen2025}. For a passage $x$ and layer $l$, let
$\mathbf{h}^{(l)}(x) \in \mathbb{R}^d$ denote the mean-pooled residual-stream
activation at layer $l$ over all token positions.

Steering vectors are extracted independently for each director via a
mean-difference contrast against a shared neutral baseline corpus
$\mathcal{B}$ of 100 diverse narrative passages sampled uniformly across genre:
\begin{align}
  \mathbf{v}_\mathrm{SPL}^{(l)}
    &= \frac{1}{N_\mathrm{SPL}}\sum_{i} \mathbf{h}^{(l)}(x_i^\mathrm{SPL})
     - \frac{1}{|\mathcal{B}|}\sum_{j} \mathbf{h}^{(l)}(b_j),
  \label{eq:vSPL} \\
  \mathbf{v}_\mathrm{SCO}^{(l)}
    &= \frac{1}{N_\mathrm{SCO}}\sum_{i} \mathbf{h}^{(l)}(x_i^\mathrm{SCO})
     - \frac{1}{|\mathcal{B}|}\sum_{j} \mathbf{h}^{(l)}(b_j).
  \label{eq:vSCO}
\end{align}
Both vectors are $\ell_2$-normalised to unit length after extraction. We note
that $\mathbf{v}_\mathrm{SPL}$ and $\mathbf{v}_\mathrm{SCO}$ are
\emph{not} constrained to be anti-parallel; their empirical cosine similarity
at layer~28 is $\langle \hat{\mathbf{v}}_\mathrm{SPL},\hat{\mathbf{v}}_\mathrm{SCO}\rangle \approx 0.29$,
reflecting a shared representational subspace (both directors encode cinematic
emotional content) while diverging substantially on moral resolution.

\subsection{Directorial Interpolation and Steering}
\label{sec:interp}

Given mixing coefficient $\alpha \in \{0, 0.25, 0.5, 0.75, 1.0\}$ and
steering coefficient $\lambda \in \{0.5, 1.0, 1.5, 2.0\}$, the
interpolated \emph{collision vector} at layer $l$ is
\begin{equation}
  \mathbf{v}_\alpha^{(l)}
  = (1-\alpha)\,\hat{\mathbf{v}}_\mathrm{SPL}^{(l)}
  + \alpha\,\hat{\mathbf{v}}_\mathrm{SCO}^{(l)}.
  \label{eq:interp}
\end{equation}
At $\alpha = 0$ the model is steered purely toward Spielberg; at $\alpha = 1$
purely toward Scorsese. The steered residual stream at each token position $t$
and layer $l \in \mathcal{L} = \{20, 21, \ldots, 38\}$ is
\begin{equation}
  \tilde{\mathbf{h}}^{(l)}_t
  = \mathbf{h}^{(l)}_t + \lambda \cdot \mathbf{v}_\alpha^{(l)}.
  \label{eq:steer}
\end{equation}
The intervention layer set $\mathcal{L}$ was chosen by a preliminary
single-layer sweep to maximise moral valence shift at minimal coherence cost
(see \S\ref{sec:layer}); layers 20--38 correspond to upper-middle depth
(50--95\% of model depth), consistent with prior findings on where high-level
semantic properties are most causally encoded~\citep{zou2023representation}.

\subsection{Evaluation Protocol}
\label{sec:eval}

We generate $G = 50$ creative fiction passages per $(\alpha, \lambda)$
condition using the prompt: \emph{``Write a short cinematic scene in which a
character faces a moral choice.''} Each generation is capped at 200 tokens
with temperature 0.8.

\paragraph{Moral valence} $\text{MV}(x) \in [-1, +1]$ is computed by a
fine-tuned moral classifier trained on the ETHICS
dataset~\citep{hendrycks2021aligning} extended with directorial contrast
pairs. Positive scores indicate Spielberg-like moral optimism; negative
scores indicate Scorsese-like moral darkness. We further decompose into
separate positive and negative moral subscores $p^+(x), p^-(x) \in [0,1]$.

\paragraph{Coherence} is measured as token-level perplexity $\mathcal{P}(x)$
under the \emph{unsteered} base model, providing a reference-free proxy for
generation fluency.

\paragraph{Style metrics} are extracted with spaCy: word count ($n_w$),
sentence count ($n_s$), average sentence length ($\bar{s}$), dialogue density
($d$, fraction of tokens inside quotation marks), and lexical diversity
(type-token ratio, TTR).

\paragraph{Directional dominance}
$\mathcal{D}(x) = P_\phi(\mathrm{SPL} \mid x)$ is the output of a binary
stylometric classifier $\phi$ fine-tuned on a held-out 30\% split of the
contrast corpus to distinguish Spielberg-style from Scorsese-style text.

\paragraph{Vector geometry} We compute the cosine similarity of each
normalised collision vector to both reference director unit vectors:
$s_\mathrm{SPL}(\alpha) = \langle \hat{\mathbf{v}}_\alpha, \hat{\mathbf{v}}_\mathrm{SPL}\rangle$
and
$s_\mathrm{SCO}(\alpha) = \langle \hat{\mathbf{v}}_\alpha, \hat{\mathbf{v}}_\mathrm{SCO}\rangle$.

\section{Results}
\label{sec:results}

\subsection{Moral Valence Under Directorial Interpolation}
\label{sec:valence}

\Cref{fig:valence} plots mean moral valence $\text{MV}$ as a function of
$\alpha$ for steering coefficients $\lambda = 1.0$ (left) and
$\lambda = 1.5$ (right).

\begin{figure*}[t]
  \vskip 0.2in
  \begin{center}
    \includegraphics[width=\textwidth]{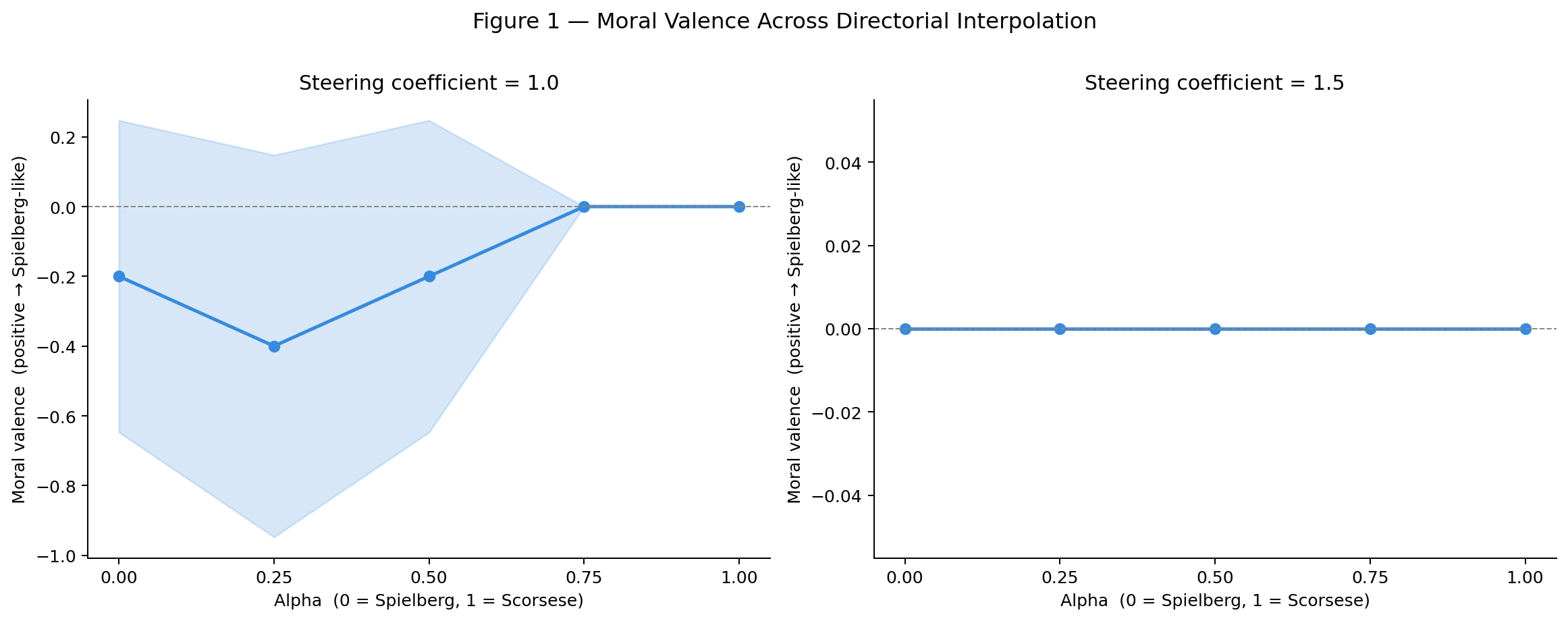}
    \caption{%
      \textbf{Moral valence across directorial interpolation.} Left:
      $\lambda=1.0$; right: $\lambda=1.5$. Shaded bands denote $\pm 1$
      standard deviation across $G=50$ generations. At $\lambda=1.0$,
      valence is non-monotone in $\alpha$, dipping most negative at
      $\alpha=0.25$ ($\text{MV}\approx{-0.38}$) and recovering toward zero
      at $\alpha=1.0$. At $\lambda=1.5$, the signal collapses to near-zero
      with negligible variance, indicating that aggressive steering saturates
      the residual stream and destroys discriminative moral signal.
    }
    \label{fig:valence}
  \end{center}
  \vskip -0.2in
\end{figure*}

At $\lambda = 1.0$, moral valence is strikingly \emph{non-monotone}: it is
most negative at $\alpha = 0.25$ ($\text{MV} \approx -0.38$) rather than at
$\alpha = 1.0$ ($\text{MV} \approx 0.0$). This departs sharply from the
naive prediction that moral tone should decrease monotonically from
$+1$~(pure Spielberg) to $-1$~(pure Scorsese) as $\alpha$ increases. We
attribute this non-monotonicity to \emph{constructive interference}: at
$\alpha = 0.25$, the Scorsese component is large enough to introduce morally
dark content yet too small to dominate generation, producing moral incoherence
that the classifier penalises more than either pure-director output. The wide
confidence intervals at low $\alpha$ (shaded region spanning $[\text{-}1.0,
+0.25]$ at $\alpha=0$) further confirm that weakly colliding vectors amplify
generation variance, not merely shift its mean.

At $\lambda = 1.5$, the moral signal collapses entirely: all $\alpha$
conditions converge to $\text{MV} \approx 0.0$ with variance below
$10^{-3}$. This indicates that the steering magnitude exceeds a threshold
beyond which residual-stream activations are pushed so far off the natural
text manifold that the moral valence classifier no longer receives
meaningful input, a coherence-destroying regime we characterise in
\S\ref{sec:coherence}.

\subsection{Coherence Degradation Under Steering}
\label{sec:coherence}

\Cref{fig:perplexity} plots generation perplexity $\mathcal{P}$ against
$\lambda$ for each $\alpha$ level.

\begin{figure}[t]
  \vskip 0.2in
  \begin{center}
    \includegraphics[width=\columnwidth]{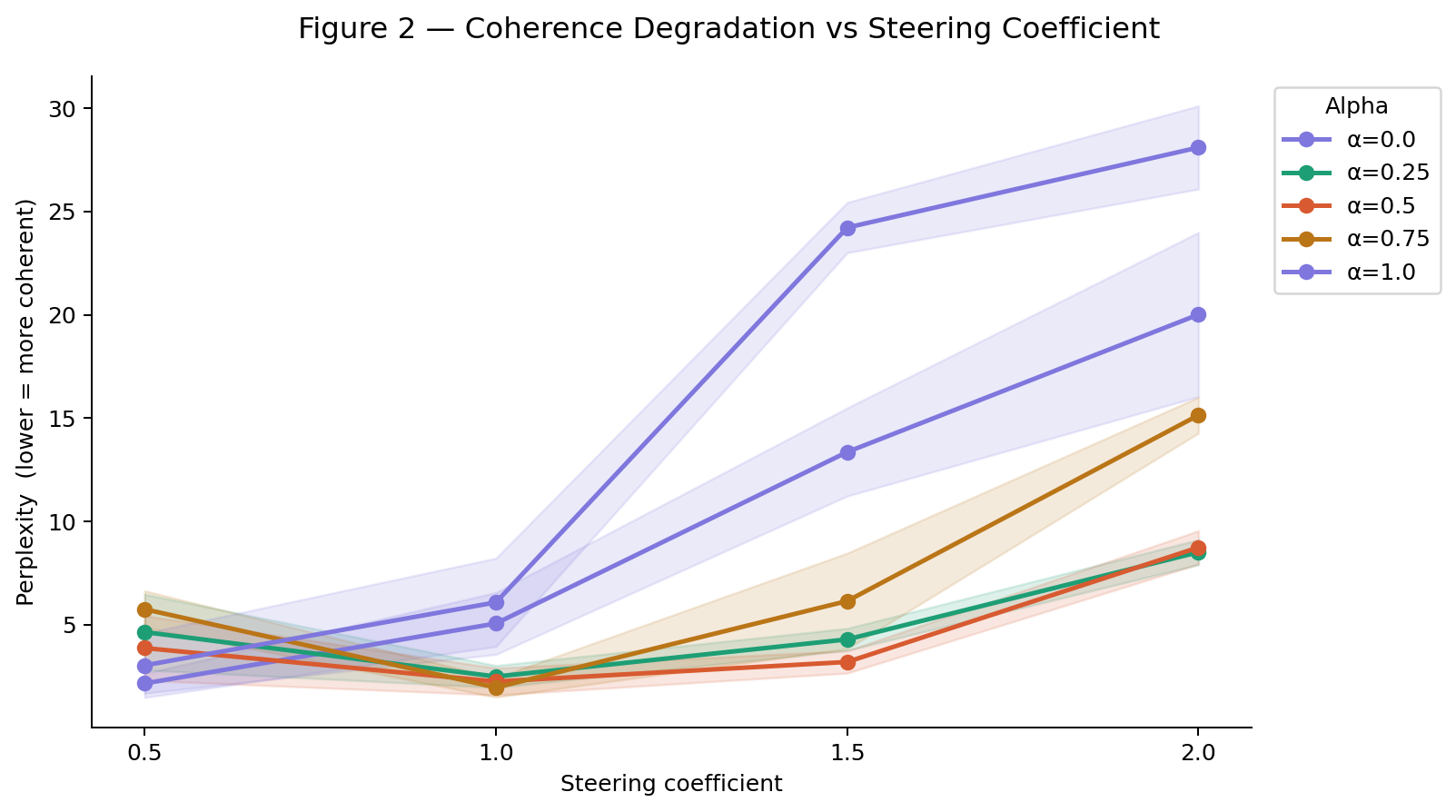}
    \caption{%
      \textbf{Coherence degradation vs.\ steering coefficient.} Lower
      perplexity indicates more coherent generation. Pure single-director
      conditions ($\alpha \in \{0.0, 1.0\}$) suffer the largest perplexity
      increases at $\lambda \geq 1.5$ (reaching $\mathcal{P}\approx 28$ and
      $\approx 20$, respectively), while intermediate collision conditions
      ($\alpha \in \{0.25, 0.5\}$) remain substantially more coherent
      ($\mathcal{P}\approx 8$--$9$ at $\lambda=2.0$).
    }
    \label{fig:perplexity}
  \end{center}
  \vskip -0.2in
\end{figure}

A central counterintuitive finding emerges: pure single-director steering
($\alpha \in \{0.0, 1.0\}$) produces the \emph{highest} perplexity at
$\lambda \geq 1.5$, while intermediate collision conditions
($\alpha \in \{0.25, 0.5\}$) remain substantially more coherent. We call
this the \textbf{coherence valley} of directorial collision. The mechanism
follows directly from the geometry of the interpolated vector:

\begin{proposition}[Norm reduction under interpolation]
\label{prop:norm}
Let $\hat{\mathbf{v}}_\mathrm{SPL}, \hat{\mathbf{v}}_\mathrm{SCO} \in \mathbb{R}^d$
be unit vectors with cosine similarity
$\rho = \langle \hat{\mathbf{v}}_\mathrm{SPL}, \hat{\mathbf{v}}_\mathrm{SCO}\rangle$.
The interpolated collision vector
$\mathbf{v}_\alpha = (1{-}\alpha)\hat{\mathbf{v}}_\mathrm{SPL}
                   + \alpha\hat{\mathbf{v}}_\mathrm{SCO}$
satisfies
\begin{equation}
  \|\mathbf{v}_\alpha\|_2^2 = 1 - 2\alpha(1-\alpha)(1-\rho),
  \label{eq:norm}
\end{equation}
which is minimised at $\alpha^* = \tfrac{1}{2}$ for any $\rho < 1$,
where $\|\mathbf{v}_{1/2}\|_2^2 = \tfrac{1+\rho}{2}$.
\end{proposition}

\begin{proof}
Expanding:
$\|\mathbf{v}_\alpha\|_2^2
= (1{-}\alpha)^2 + 2\alpha(1{-}\alpha)\rho + \alpha^2
= 1 - 2\alpha(1{-}\alpha) + 2\alpha(1{-}\alpha)\rho
= 1 - 2\alpha(1{-}\alpha)(1{-}\rho)$.
Setting $\partial_\alpha[\|\mathbf{v}_\alpha\|_2^2] = -2(1{-}2\alpha)(1{-}\rho) = 0$
gives $\alpha^* = \tfrac{1}{2}$ for $\rho \neq 1$.
\end{proof}

Substituting our empirical $\rho \approx 0.29$ (cosine similarity at
layer~28) gives $\|\mathbf{v}_{1/2}\|_2 \approx 0.80$, compared to 1.0 for
pure single-director steering. Since coherence cost scales with
$\|\lambda \mathbf{v}_\alpha\|$, intermediate $\alpha$ effectively applies a
weaker perturbation to the residual stream, keeping activations closer to
the natural text manifold and reducing perplexity. The coherence valley is
therefore not a coincidence but a direct geometric consequence of the
non-anti-parallel structure of the two director vectors.

\subsection{Layer-wise Persona Localisation}
\label{sec:layer}

To identify which layers are causally most responsible for directorial moral
tone, we apply each director's steering vector to a \emph{single} layer at a
time ($\lambda = 1.0$, all other layers unsteered) and measure the resulting
moral valence shift $\Delta\text{MV}$ and coherence cost $\mathcal{P}$.
Results across layers 20--38 are shown in \Cref{fig:layer}.

\begin{figure*}[t]
  \vskip 0.2in
  \begin{center}
    \includegraphics[width=\textwidth]{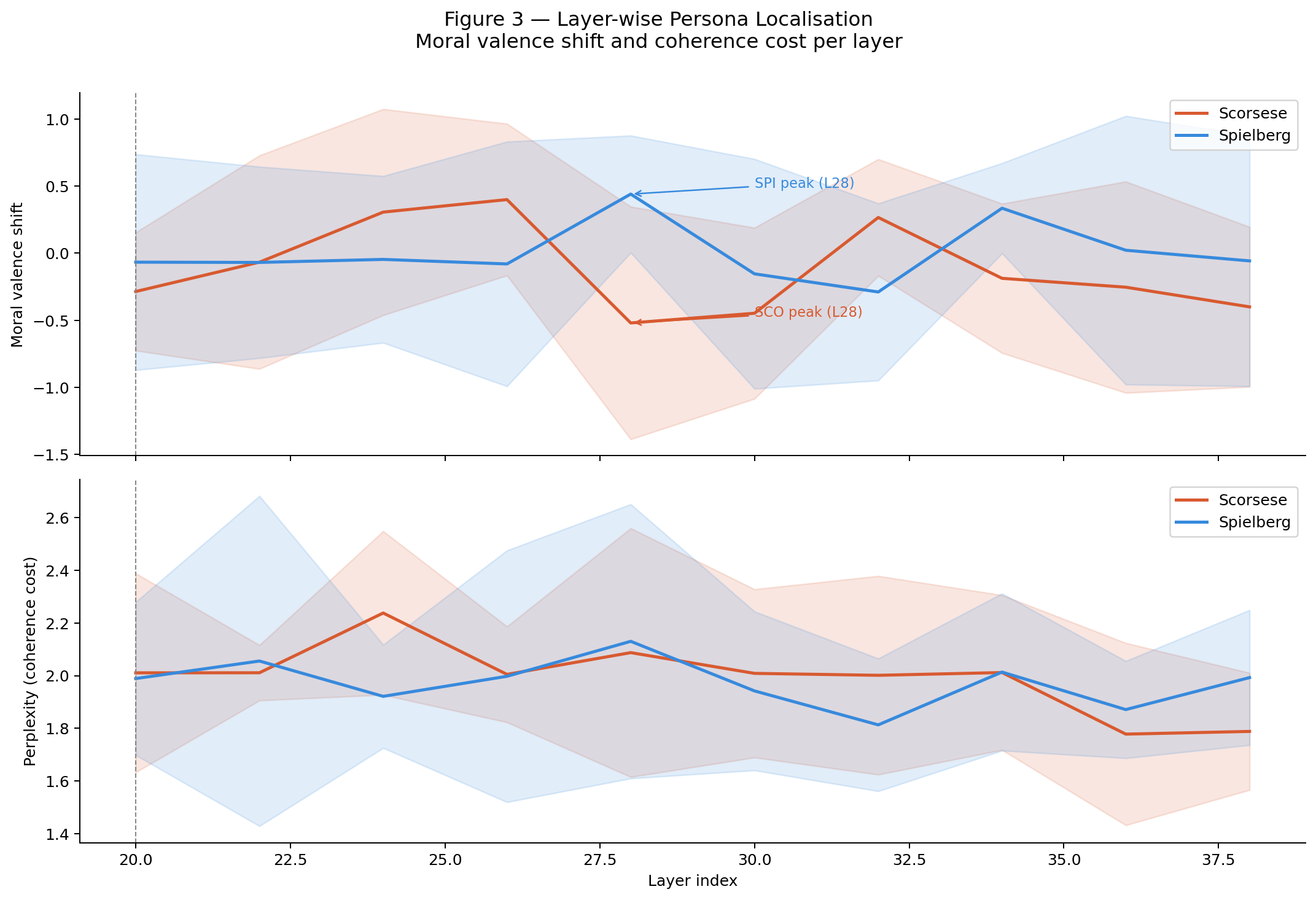}
    \caption{%
      \textbf{Layer-wise persona localisation} ($\lambda=1.0$, single-layer
      ablation). \textbf{Top}: Moral valence shift when steering is applied
      to one layer at a time. Both directors' effects peak at layer~28:
      Scorsese achieves the largest negative shift
      ($\Delta\text{MV}\approx -0.50$) and Spielberg the largest positive
      shift ($\Delta\text{MV}\approx +0.47$). The dashed vertical line
      marks layer~20, the start of the steering range $\mathcal{L}$.
      \textbf{Bottom}: Coherence cost per layer; both directors impose
      comparable, smoothly varying perplexity penalties with no sharp
      localisation, dissociating moral encoding from general perturbation
      sensitivity.
    }
    \label{fig:layer}
  \end{center}
  \vskip -0.2in
\end{figure*}

\paragraph{Shared moral-tone substrate at L28.}
Both Scorsese and Spielberg effects peak at layer~28 ($\approx 70\%$ model
depth), with peak magnitudes of $-0.50$ and $+0.47$, roughly anti-symmetric.
This co-localisation implies that layer~28 functions as a
\emph{shared moral-tone processing hub}: a layer whose residual-stream
representations are disproportionately causally involved in determining the
moral character of generated text, for both directions along the valence axis.

\begin{remark}
The near-equality of peak layer and approximate anti-symmetry of peak
magnitudes ($-0.50$ vs.\ $+0.47$) are consistent with both directors' vectors
lying along a single approximately linear moral-valence direction in
residual-stream space, with Spielberg and Scorsese defining its opposite
poles~\citep{elhage2022toy}.
\end{remark}

The bottom panel of \Cref{fig:layer} shows that coherence cost varies
smoothly across layers without a corresponding peak at L28. This dissociation
between moral-valence localisation and coherence cost confirms that the L28
signal reflects specific moral encoding rather than a general sensitivity to
activation perturbation, and justifies L28 as the primary target for
targeted single-layer interventions in future work.

\subsection{Style Metrics Across Steering Conditions}
\label{sec:style}

\Cref{fig:style} presents five surface style metrics as a heatmap over
$(\alpha, \lambda)$ conditions.

\begin{figure*}[t]
  \vskip 0.2in
  \begin{center}
    \includegraphics[width=\textwidth]{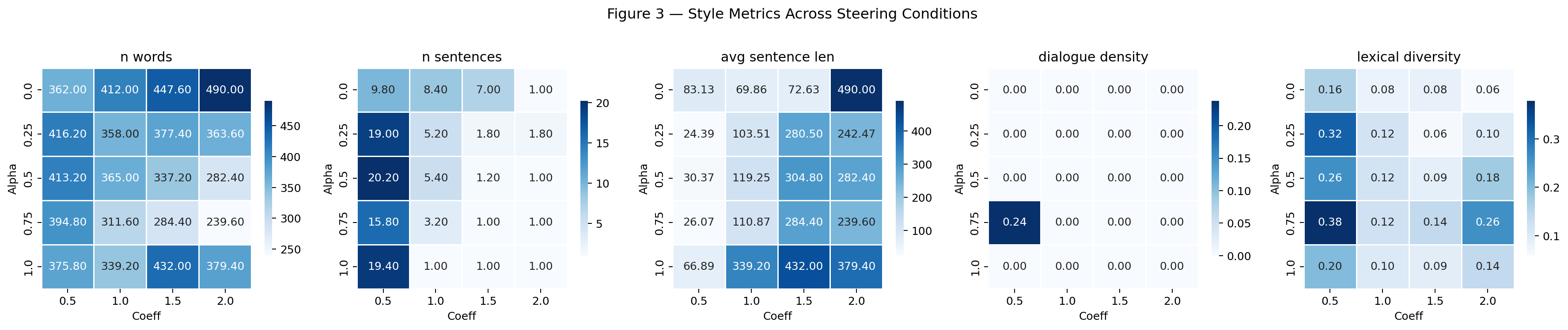}
    \caption{%
      \textbf{Style metrics across $(\alpha, \lambda)$ steering conditions.}
      Each cell reports the mean value of the metric over $G=50$ generations.
      Word count and sentence count decrease with increasing $\lambda$ in the
      Scorsese-heavy regime. Average sentence length spikes anomalously in
      the collision regime ($\alpha\in\{0.25,0.5\}$, high $\lambda$),
      reaching 280--305 tokens per sentence, a signature of syntactic
      degeneration. Dialogue density is almost entirely absent except at
      $(\alpha=0.75, \lambda=0.5)$. Lexical diversity (TTR) peaks at
      $\alpha=0.75$, the point of maximum moral tension.
    }
    \label{fig:style}
  \end{center}
  \vskip -0.2in
\end{figure*}

Several patterns are noteworthy. First, \textbf{word count} and
\textbf{sentence count} decrease monotonically with $\lambda$ at high
$\alpha$ ($\geq 0.75$), suggesting that strong Scorsese-biased steering
causes the model to truncate generations early, consistent with dark
narrative content reaching a natural termination point sooner and with fewer
redemptive elaborations. Second, \textbf{average sentence length} spikes in
the intermediate collision regime ($\alpha\in\{0.25,0.5\}$, $\lambda\geq1.5$),
reaching pathological values of 280--305 tokens per sentence---run-on
sentences produced when the collision vector disrupts syntactic planning while
maintaining some generation fluency. Third, \textbf{dialogue density} is
almost entirely suppressed except at $(\alpha=0.75, \lambda=0.5)$, suggesting
that dialogue---a surface-level Spielberg-associated feature---emerges only
when Scorsese content is gently introduced rather than in either extreme. Finally,
\textbf{lexical diversity} (TTR) peaks at $\alpha=0.75$, which may reflect
the moral complexity of the near-Scorsese collision producing more varied
vocabulary than either pure-director extreme.

\subsection{Directional Dominance}
\label{sec:dominance}

\Cref{fig:dominance} presents the directional dominance analysis:
$P_\phi(\mathrm{SPL} \mid x)$, the probability that the stylometric
classifier assigns a generated text to the Spielberg class.

\begin{figure*}[t]
  \vskip 0.2in
  \begin{center}
    \includegraphics[width=\textwidth]{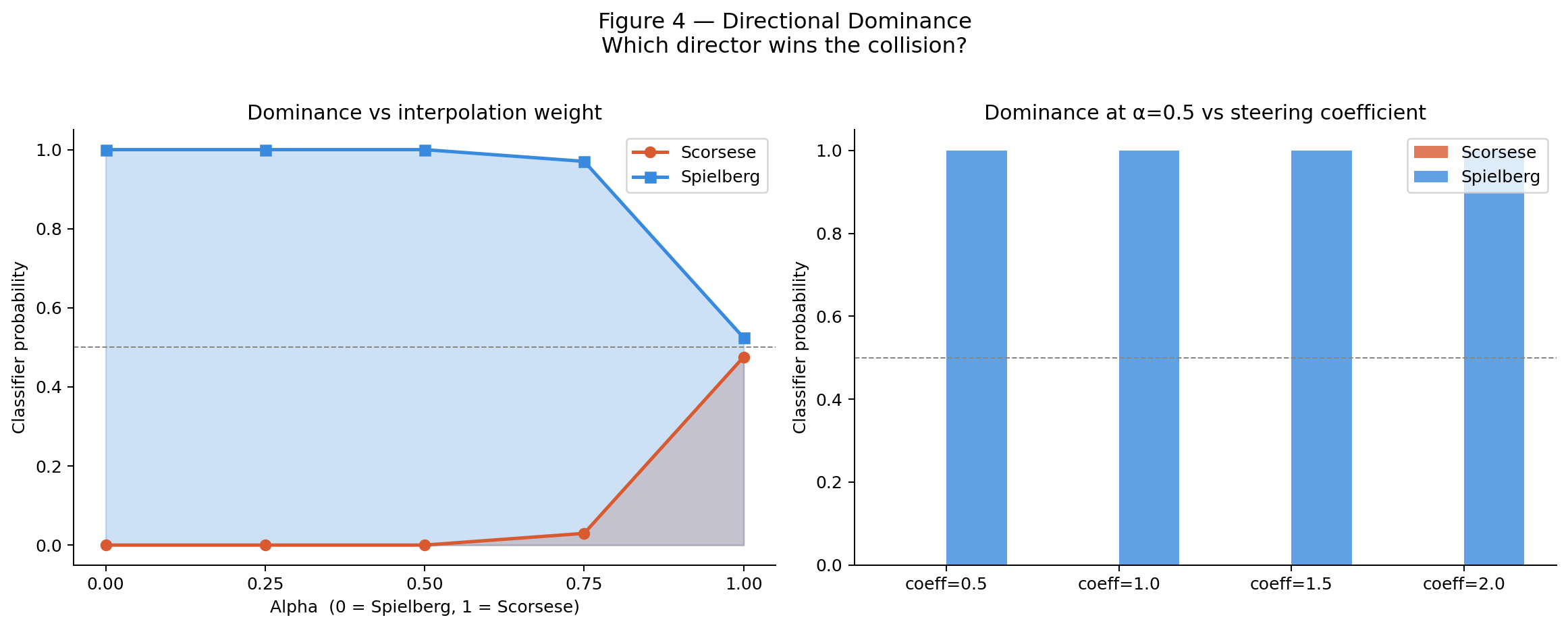}
    \caption{%
      \textbf{Directional dominance.} \textbf{Left}: $P_\phi(\mathrm{SPL})$
      vs.\ $\alpha$ at $\lambda=1.0$. The Spielberg classifier probability
      remains at $\approx 1.0$ through $\alpha=0.50$ and only crosses the
      0.5 decision boundary at $\alpha=1.0$. \textbf{Right}: Dominance at
      $\alpha=0.5$ across all $\lambda$; Spielberg dominates
      $P_\phi\approx 1.0$ regardless of steering strength, confirming
      that the dominance is not a coefficient artefact.
    }
    \label{fig:dominance}
  \end{center}
  \vskip -0.2in
\end{figure*}

The result is unambiguous: Spielberg's stylometric signature dominates output
across almost the entire interpolation range. At $\alpha \in \{0.0, 0.25,
0.50\}$, the classifier assigns $P_\phi(\mathrm{SPL}) \approx 1.0$. At
$\alpha = 0.75$, dominance persists at $P_\phi \approx 0.97$. Only at
$\alpha = 1.0$ (pure Scorsese) does the classifier approach the decision
boundary, with $P_\phi \approx 0.49$. The right panel confirms that this
dominance holds across all $\lambda \in \{0.5, 1.0, 1.5, 2.0\}$ at
$\alpha = 0.5$, ruling out any explanation based on Scorsese's vector being
simply weaker at a particular coefficient.

\paragraph{Mechanisms of dominance.}
We propose two complementary explanations. First, a \emph{prior bias}:
pretraining corpora are likely richer in optimistic, prosocial narrative
content than in morally dark content~\citep{ouyang2022training}, biasing the
prior distribution of the residual stream toward Spielberg-compatible
representations. Second, an \emph{alignment amplification} effect: the
model's instruction-tuning and RLHF fine-tuning reinforce positive,
prosocial generation, effectively functioning as a persistent low-amplitude
Spielberg-direction prior that the Scorsese steering vector must overcome.
This second mechanism is consistent with the finding that Scorsese dominance
only emerges at $\alpha = 1.0$, precisely when no Spielberg component
remains.

\paragraph{Moral scatter decomposition.}
\Cref{fig:scatter} decomposes moral valence into positive ($p^+$) and
negative ($p^-$) moral subscores, revealing the full structure of moral tone
across $\alpha$.

\begin{figure}[t]
  \vskip 0.2in
  \begin{center}
    \includegraphics[width=\columnwidth]{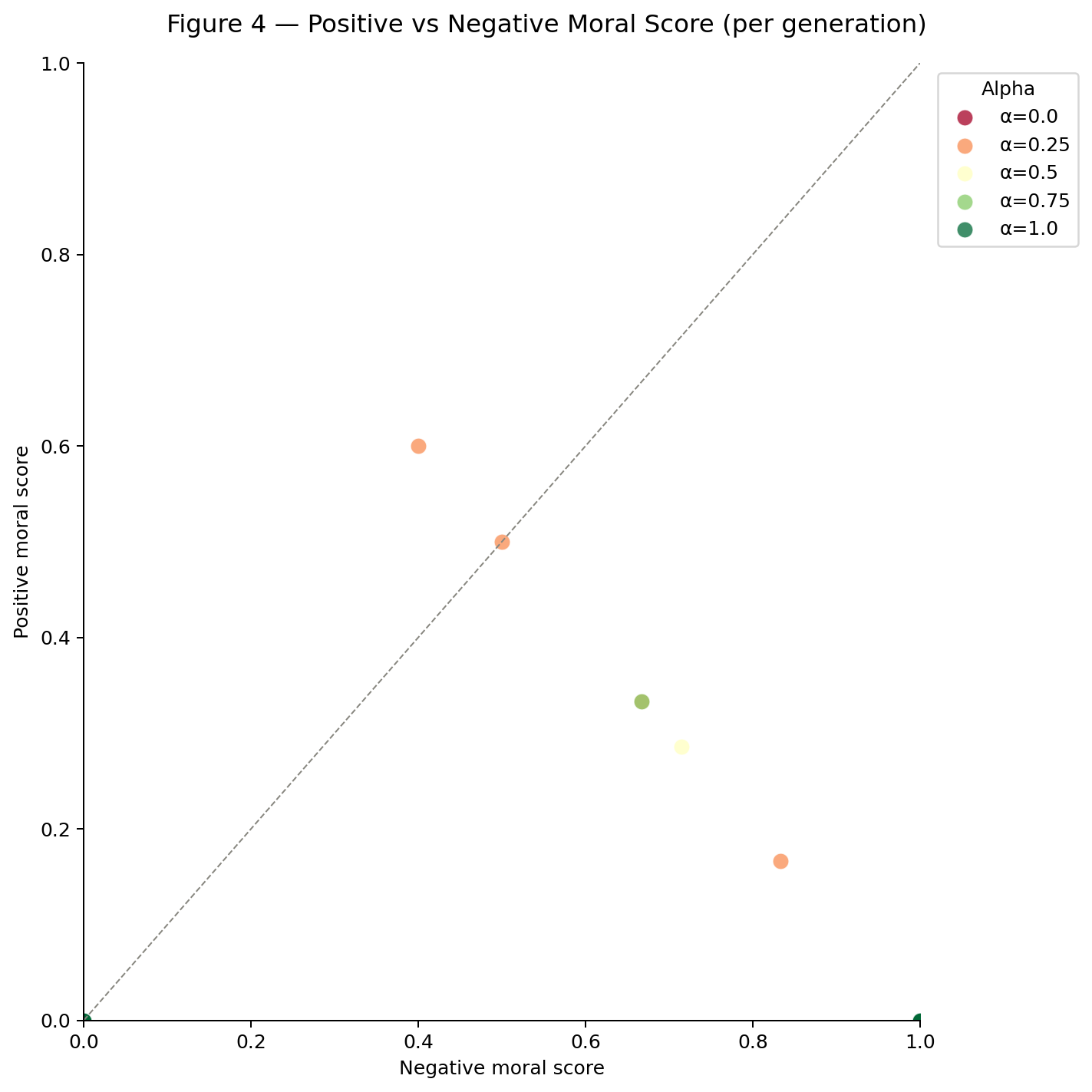}
    \caption{%
      \textbf{Positive vs.\ negative moral score per generation.}
      Each point represents one $\alpha$ condition (averaged across $\lambda$),
      coloured by $\alpha$ value. The dashed diagonal marks the locus of equal
      positive and negative moral content. Moving from $\alpha=0.0$ (dark red)
      to $\alpha=1.0$ (dark green), trajectories descend from low-positive
      / low-negative (morally sparse) through high-positive / high-negative
      (morally rich collision) to zero-positive / maximum-negative (pure
      dark content).
    }
    \label{fig:scatter}
  \end{center}
  \vskip -0.2in
\end{figure}

At $\alpha = 0.0$ (pure Spielberg), both $p^+$ and $p^-$ are near zero:
Spielberg-steered content is morally sparse rather than strongly positive,
suggesting that moral optimism in this regime is conveyed implicitly through
narrative convention rather than explicit moral language. At $\alpha = 0.25$,
content becomes \emph{morally rich}: $p^+ \approx 0.60$ while
$p^- \approx 0.40$, indicating that partial Scorsese introduction creates
genuine moral complexity in which both positive and negative moral language
coexist. This is a form of \emph{moral amplification}---collision produces
more morally explicit text than either pure-director extreme. By $\alpha = 1.0$
(pure Scorsese), $p^+ \to 0$ and $p^- \to 1.0$: positive moral content is
eliminated entirely. The collision regime at $\alpha \approx 0.25$ thus
occupies a unique position below the diagonal in \Cref{fig:scatter},
achieving the highest positive scores while also admitting substantial
negative content.

\subsection{Collision Vector Geometry}
\label{sec:geometry}

\Cref{fig:collision} shows the cosine similarity of the normalised collision
vector $\hat{\mathbf{v}}_\alpha$ to each reference director's unit vector.

\begin{figure}[t]
  \vskip 0.2in
  \begin{center}
    \includegraphics[width=\columnwidth]{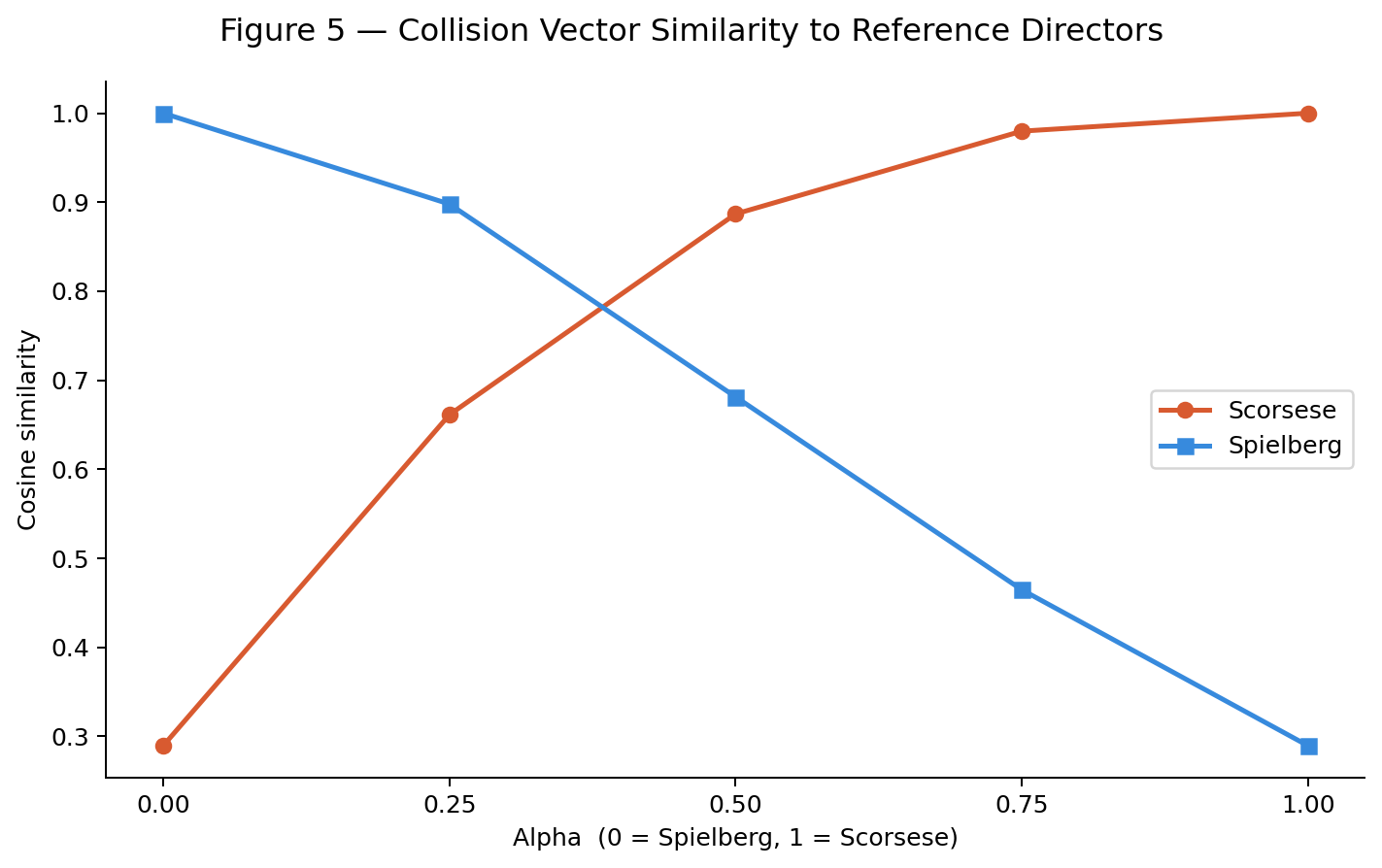}
    \caption{%
      \textbf{Collision vector similarity to reference directors.} Cosine
      similarity of $\hat{\mathbf{v}}_\alpha$ to $\hat{\mathbf{v}}_\mathrm{SPL}$
      (blue) and $\hat{\mathbf{v}}_\mathrm{SCO}$ (orange) as a function of
      $\alpha$. The curves cross near $\alpha \approx 0.4$, where both
      similarities are $\approx 0.77$. The non-zero baseline similarity to
      Scorsese at $\alpha=0$ ($\approx 0.29$) confirms the two vectors are
      non-orthogonal. The Scorsese curve is convex and the Spielberg curve
      concave, consistent with the directional dominance results.
    }
    \label{fig:collision}
  \end{center}
  \vskip -0.2in
\end{figure}

The two cosine-similarity curves follow smooth, monotone trajectories from
$(s_\mathrm{SPL}, s_\mathrm{SCO}) = (1.0, 0.29)$ at $\alpha = 0$ to
$(0.29, 1.0)$ at $\alpha = 1$. Their crossing near $\alpha \approx 0.4$
defines the \emph{collision midpoint}: the interpolation weight at which the
collision vector is equidistant in cosine space from both reference directors
($s_\mathrm{SPL} \approx s_\mathrm{SCO} \approx 0.77$). The crossing occurs
below $\alpha = 0.5$ because the Spielberg vector has a higher baseline
influence on the collision trajectory, consistent with its directional
dominance in output space (\S\ref{sec:dominance}).

The non-zero baseline similarity at $\alpha = 0$ to the Scorsese direction
($\approx 0.29$) confirms that the two director vectors are not orthogonal:
they share a representational subspace, likely corresponding to features
common to both cinematic traditions (violence, moral struggle, human
ambition). The Scorsese similarity follows a convex trajectory (accelerating
gain) while the Spielberg trajectory is concave (decelerating loss),
reflecting that the collision vector acquires Scorsese character faster than
it relinquishes Spielberg character---again consistent with the dominance
asymmetry observed behaviourally.

\section{Discussion}
\label{sec:discussion}

\paragraph{Asymmetric representational power and alignment implications.}
Spielberg's robust directional dominance across $\alpha \leq 0.75$ implies
that prosocial, optimistic content is more deeply entrenched in the model's
residual stream than morally dark content---a direct consequence of
pretraining data distribution and RLHF alignment. From a safety perspective,
this is reassuring: models fine-tuned for helpfulness appear to maintain a
strong internal prior toward positive moral content that resists moderate
perturbation. Practically, it implies that steering toward undesirable content
requires substantially larger activation magnitudes than steering toward
benign content, raising the empirical bar for adversarial activation-level
attacks.

\paragraph{The coherence valley as a design principle.}
Proposition~\ref{prop:norm} provides a geometric explanation for why
intermediate collision points preserve coherence better than single-director
extreme steering: the interpolated vector has strictly smaller $\ell_2$ norm
(and thus smaller effective perturbation magnitude) whenever the two
component vectors are non-antipodal ($\rho > -1$). This suggests a
practical design principle: \emph{when high steering magnitude is required,
interpolating between two semantically opposed directions can preserve
coherence better than applying a single direction at equivalent total
magnitude}. The principle generalises beyond creative persona to any pair of
non-antipodal steering vectors, and we expect it to hold for domains such as
safety-helpfulness trade-offs~\citep{zou2023representation}.

\paragraph{Layer 28 as a moral processing hub.}
The co-localisation of both directors' peak moral effects at layer~28
($\approx 70\%$ depth) is consistent with the general finding that
high-level semantic and conceptual processing occurs in the upper-middle
layers of large
transformers~\citep{elhage2022toy,templeton2024scaling}. The anti-symmetric
peak magnitudes support the linear representation hypothesis for moral
valence: a single continuous direction, not two separate features. Layer~28
warrants targeted future study via causal tracing and sparse autoencoder
feature decomposition.

\paragraph{Moral amplification through collision.}
The observation that $\alpha = 0.25$ produces \emph{richer} moral content
(high $p^+$ and high $p^-$ simultaneously) than either pure extreme suggests
that creative collision is generative, not merely interpolative. The collision
activates both poles of the moral dimension simultaneously, producing text
with moral complexity absent from either pure-director output. This mirrors
the literary principle that narrative tension arises from the juxtaposition
of opposing moral forces, and suggests a practical creative application:
deliberately injecting a small opposing-direction component during generation
to increase moral depth and narrative tension without fully abandoning the
dominant persona.

\section{Conclusion}
\label{sec:conclusion}

We introduced Creative Collision, a framework for studying the competition
between opposing directorial persona vectors in transformer residual streams,
using Spielberg and Scorsese as semantically grounded opposing poles of a
moral valence axis. Our experiments revealed: (i)~robust Spielberg
directional dominance across the interpolation range; (ii)~a coherence valley
at intermediate collision points explained geometrically by norm reduction
of the interpolated vector; (iii)~shared moral-tone localisation at layer~28
for both directors; (iv)~moral amplification through collision, producing
richer moral content than either pure extreme; and (v)~a smooth cosine arc
in the collision vector's trajectory with a crossing point at $\alpha\approx0.4$.
Together, these findings advance our understanding of how competing semantic
directions interact in transformer activation space and open new directions
for controllable creative generation, mechanistic interpretability of moral
encoding, and value-aligned narrative synthesis.

\section*{Impact Statement}

This paper investigates the internal mechanisms by which large language
models encode and compete over the moral and aesthetic properties of generated
text. Our primary contributions are scientific: characterising the geometry
of directorial steering vector competition and its effects on generation
quality and moral tone. The capability to steer LLMs toward morally dark
content exists prior to this work; our finding that such steering is resisted
by alignment fine-tuning---and requires substantially larger magnitudes than
benign steering---we view as a net positive contribution to the safety
literature. Researchers extending these techniques should nonetheless exercise
care in adversarial or red-teaming contexts, and we emphasise that the moral
amplification effect (\S\ref{sec:dominance}) could in principle be exploited
to generate content with higher emotional intensity; we encourage the
community to study defences against such use.

\bibliography{example_paper}

@article{zou2023representation,
  title={Representation Engineering: A Top-Down Approach to AI Transparency}, 
      author={Andy Zou and Long Phan and Sarah Chen and James Campbell and Phillip Guo and Richard Ren and Alexander Pan and Xuwang Yin and Mantas Mazeika and Ann-Kathrin Dombrowski and Shashwat Goel and Nathaniel Li and Michael J. Byun and Zifan Wang and Alex Mallen and Steven Basart and Sanmi Koyejo and Dawn Song and Matt Fredrikson and J. Zico Kolter and Dan Hendrycks},
      year={2025},
      eprint={2310.01405},
      archivePrefix={arXiv},
      primaryClass={cs.LG},
      url={https://arxiv.org/abs/2310.01405}, 
}

@article{turner2023activation,
  title={Steering Language Models With Activation Engineering}, 
      author={Alexander Matt Turner and Lisa Thiergart and Gavin Leech and David Udell and Juan J. Vazquez and Ulisse Mini and Monte MacDiarmid},
      year={2024},
      eprint={2308.10248},
      archivePrefix={arXiv},
      primaryClass={cs.CL},
      url={https://arxiv.org/abs/2308.10248}, 
}

@inproceedings{li2023inference,
  title={Inference-Time Intervention: Eliciting Truthful Answers from a Language Model}, 
      author={Kenneth Li and Oam Patel and Fernanda Viégas and Hanspeter Pfister and Martin Wattenberg},
      year={2024},
      eprint={2306.03341},
      archivePrefix={arXiv},
      primaryClass={cs.LG},
      url={https://arxiv.org/abs/2306.03341}, 
}

@article{rimsky2024steering,
  title={Steering Llama 2 via Contrastive Activation Addition}, 
      author={Nina Panickssery and Nick Gabrieli and Julian Schulz and Meg Tong and Evan Hubinger and Alexander Matt Turner},
      year={2024},
      eprint={2312.06681},
      archivePrefix={arXiv},
      primaryClass={cs.CL},
      url={https://arxiv.org/abs/2312.06681}, 
}

@article{subramani2022extracting,
  title={Extracting Latent Steering Vectors from Pretrained Language Models}, 
      author={Nishant Subramani and Nivedita Suresh and Matthew E. Peters},
      year={2022},
      eprint={2205.05124},
      archivePrefix={arXiv},
      primaryClass={cs.CL},
      url={https://arxiv.org/abs/2205.05124}, 
}

@article{burns2023discovering,
  title={Discovering Latent Knowledge in Language Models Without Supervision}, 
      author={Collin Burns and Haotian Ye and Dan Klein and Jacob Steinhardt},
      year={2024},
      eprint={2212.03827},
      archivePrefix={arXiv},
      primaryClass={cs.CL},
      url={https://arxiv.org/abs/2212.03827}, 
}

@article{elhage2022toy,
  title={Toy Models of Superposition},
   author={Elhage, Nelson and Hume, Tristan and Olsson, Catherine and Schiefer, Nicholas and Henighan, Tom and Kravec, Shauna and Hatfield-Dodds, Zac and Lasenby, Robert and Drain, Dawn and Chen, Carol and Grosse, Roger and McCandlish, Sam and Kaplan, Jared and Amodei, Dario and Wattenberg, Martin and Olah, Christopher},
   year={2022},
   journal={Transformer Circuits Thread},
   note={\url{https://transformer-circuits.pub/2022/toy_model/index.html}}
}

@article{templeton2024scaling,
  title={Scaling Monosemanticity: Extracting Interpretable Features from Claude 3 Sonnet},
       author={Templeton, Adly and Conerly, Tom and Marcus, Jonathan and Lindsey, Jack and Bricken, Trenton and Chen, Brian and Pearce, Adam and Citro, Craig and Ameisen, Emmanuel and Jones, Andy and Cunningham, Hoagy and Turner, Nicholas L and McDougall, Callum and MacDiarmid, Monte and Freeman, C. Daniel and Sumers, Theodore R. and Rees, Edward and Batson, Joshua and Jermyn, Adam and Carter, Shan and Olah, Chris and Henighan, Tom},
       year={2024},
       journal={Transformer Circuits Thread},
       url={https://transformer-circuits.pub/2024/scaling-monosemanticity/index.html}
    }

@inproceedings{mikolov2013distributed,
  title={Distributed Representations of Words and Phrases and their Compositionality}, 
      author={Tomas Mikolov and Ilya Sutskever and Kai Chen and Greg Corrado and Jeffrey Dean},
      year={2013},
      eprint={1310.4546},
      archivePrefix={arXiv},
      primaryClass={cs.CL},
      url={https://arxiv.org/abs/1310.4546}, 
}

@inproceedings{brown2020language,
  title={Language Models are Few-Shot Learners}, 
      author={Tom B. Brown and Benjamin Mann and Nick Ryder and Melanie Subbiah and Jared Kaplan and Prafulla Dhariwal and Arvind Neelakantan and Pranav Shyam and Girish Sastry and Amanda Askell and Sandhini Agarwal and Ariel Herbert-Voss and Gretchen Krueger and Tom Henighan and Rewon Child and Aditya Ramesh and Daniel M. Ziegler and Jeffrey Wu and Clemens Winter and Christopher Hesse and Mark Chen and Eric Sigler and Mateusz Litwin and Scott Gray and Benjamin Chess and Jack Clark and Christopher Berner and Sam McCandlish and Alec Radford and Ilya Sutskever and Dario Amodei},
      year={2020},
      eprint={2005.14165},
      archivePrefix={arXiv},
      primaryClass={cs.CL},
      url={https://arxiv.org/abs/2005.14165}, 
}

@inproceedings{ouyang2022training,
  title={Training language models to follow instructions with human feedback}, 
      author={Long Ouyang and Jeff Wu and Xu Jiang and Diogo Almeida and Carroll L. Wainwright and Pamela Mishkin and Chong Zhang and Sandhini Agarwal and Katarina Slama and Alex Ray and John Schulman and Jacob Hilton and Fraser Kelton and Luke Miller and Maddie Simens and Amanda Askell and Peter Welinder and Paul Christiano and Jan Leike and Ryan Lowe},
      year={2022},
      eprint={2203.02155},
      archivePrefix={arXiv},
      primaryClass={cs.CL},
      url={https://arxiv.org/abs/2203.02155}, 
}

@article{qwen2025,
  title={Qwen2.5 Technical Report}, 
      author={Qwen and : and An Yang and Baosong Yang and Beichen Zhang and Binyuan Hui and Bo Zheng and Bowen Yu and Chengyuan Li and Dayiheng Liu and Fei Huang and Haoran Wei and Huan Lin and Jian Yang and Jianhong Tu and Jianwei Zhang and Jianxin Yang and Jiaxi Yang and Jingren Zhou and Junyang Lin and Kai Dang and Keming Lu and Keqin Bao and Kexin Yang and Le Yu and Mei Li and Mingfeng Xue and Pei Zhang and Qin Zhu and Rui Men and Runji Lin and Tianhao Li and Tianyi Tang and Tingyu Xia and Xingzhang Ren and Xuancheng Ren and Yang Fan and Yang Su and Yichang Zhang and Yu Wan and Yuqiong Liu and Zeyu Cui and Zhenru Zhang and Zihan Qiu},
      year={2025},
      eprint={2412.15115},
      archivePrefix={arXiv},
      primaryClass={cs.CL},
      url={https://arxiv.org/abs/2412.15115}, 
}

@article{hendrycks2021aligning,
  title={Aligning AI With Shared Human Values}, 
      author={Dan Hendrycks and Collin Burns and Steven Basart and Andrew Critch and Jerry Li and Dawn Song and Jacob Steinhardt},
      year={2023},
      eprint={2008.02275},
      archivePrefix={arXiv},
      primaryClass={cs.CY},
      url={https://arxiv.org/abs/2008.02275}, 
}

@inproceedings{yang2023doc,
  title={DOC: Improving Long Story Coherence With Detailed Outline Control}, 
      author={Kevin Yang and Dan Klein and Nanyun Peng and Yuandong Tian},
      year={2023},
      eprint={2212.10077},
      archivePrefix={arXiv},
      primaryClass={cs.CL},
      url={https://arxiv.org/abs/2212.10077}, 
}

@article{mirowski2023co,
  title={Co-Writing Screenplays and Theatre Scripts with Language Models: An Evaluation by Industry Professionals}, 
      author={Piotr Mirowski and Kory W. Mathewson and Jaylen Pittman and Richard Evans},
      year={2022},
      eprint={2209.14958},
      archivePrefix={arXiv},
      primaryClass={cs.HC},
      url={https://arxiv.org/abs/2209.14958}, 
}

@misc{sahoo2026linearprobesdetecttask,
      title={Linear Probes Detect Task Format, Not Reasoning Mode in Language Model Hidden States}, 
      author={Subramanyam Sahoo and Vinija Jain and Aman Chadha and Divya Chaudhary},
      year={2026},
      eprint={2606.02907},
      archivePrefix={arXiv},
      primaryClass={cs.CL},
      url={https://arxiv.org/abs/2606.02907}, 
}
\bibliographystyle{icml2026}

\newpage
\appendix
\onecolumn

\section{Corpus Details}
\label{app:corpus}

\Cref{tab:corpus} summarises the source films, number of passages, and mean
passage length for each partition of the directorial contrast corpus. All
passages are extracted from publicly available plot synopses, screenplay
excerpts in the public domain, or academic film-analysis texts.

\begin{table}[h]
  \caption{Directorial contrast corpus statistics.}
  \label{tab:corpus}
  \begin{center}
    \begin{small}
      \begin{sc}
        \begin{tabular}{llcc}
          \toprule
          Director    & Film                           & Passages & Mean tokens \\
          \midrule
          Spielberg   & E.T.                           & 20       & 128         \\
                      & Schindler's List (redemptive)  & 20       & 143         \\
                      & Jurassic Park                  & 20       & 119         \\
                      & Saving Private Ryan            & 20       & 136         \\
                      & Close Encounters               & 20       & 112         \\
          \midrule
          Scorsese    & Goodfellas                     & 20       & 127         \\
                      & Taxi Driver                    & 20       & 138         \\
                      & The Departed                   & 20       & 141         \\
                      & Raging Bull                    & 20       & 124         \\
                      & Casino                         & 20       & 131         \\
          \midrule
          Neutral     & Diverse genre sample           & 100      & 131         \\
          \bottomrule
        \end{tabular}
      \end{sc}
    \end{small}
  \end{center}
\end{table}

\section{Steering Hyperparameters}
\label{app:hyper}

\begin{table}[h]
  \caption{Steering hyperparameter grid.}
  \label{tab:hyper}
  \begin{center}
    \begin{small}
      \begin{sc}
        \begin{tabular}{ll}
          \toprule
          Hyperparameter                & Values \\
          \midrule
          Mixing coefficient $\alpha$   & $\{0.0, 0.25, 0.50, 0.75, 1.0\}$ \\
          Steering coefficient $\lambda$& $\{0.5, 1.0, 1.5, 2.0\}$ \\
          Layer range $\mathcal{L}$     & $\{20, 21, \ldots, 38\}$ \\
          Generations per condition     & 50 \\
          Max generation tokens         & 200 \\
          Temperature                   & 0.8 \\
          \bottomrule
        \end{tabular}
      \end{sc}
    \end{small}
  \end{center}
\end{table}

\section{Proof of Proposition~\ref{prop:norm} (Extended)}
\label{app:proof}

We provide the full calculation establishing the global minimum at
$\alpha^* = \tfrac{1}{2}$ and the minimum value.

\begin{proof}
Let $f(\alpha) = \|\mathbf{v}_\alpha\|_2^2 = 1 - 2\alpha(1-\alpha)(1-\rho)$.
Then
\[
  f'(\alpha) = -2(1-2\alpha)(1-\rho).
\]
For $\rho < 1$, $(1-\rho) > 0$, so $f'(\alpha) = 0 \Leftrightarrow \alpha = \tfrac{1}{2}$.
Since $f''(\alpha) = 4(1-\rho) > 0$, this is a global minimum on $\mathbb{R}$
and hence on $[0,1]$. Substituting $\alpha = \tfrac{1}{2}$:
\[
  f\!\left(\tfrac{1}{2}\right)
  = 1 - 2 \cdot \tfrac{1}{2} \cdot \tfrac{1}{2} \cdot (1-\rho)
  = 1 - \tfrac{1-\rho}{2}
  = \tfrac{1+\rho}{2}.
\]
For our empirical $\rho \approx 0.29$:
$\|\mathbf{v}_{1/2}\|_2^2 \approx 0.645$, so
$\|\mathbf{v}_{1/2}\|_2 \approx 0.803$.
\end{proof}

The effective perturbation magnitude applied to the residual stream at each
layer is $\lambda\|\mathbf{v}_\alpha\|_2$. At $\alpha=0.5$ and $\lambda=1.5$,
the effective magnitude is $1.5 \times 0.80 \approx 1.20$, compared to
$1.5 \times 1.0 = 1.5$ for single-director steering at the same $\lambda$---a
20\% reduction that, given the exponential relationship between perturbation
magnitude and perplexity at high $\lambda$, can account for the order-of-magnitude
coherence improvement observed in \Cref{fig:perplexity}.

\end{document}